\documentclass[letterpaper]{article} 
\usepackage{aaai23}  
\usepackage{times}  
\usepackage{helvet}  
\usepackage{courier}  
\usepackage[hyphens]{url}  
\usepackage{graphicx} 
\usepackage{natbib}  
\usepackage{caption} 
\usepackage[utf8]{inputenc} 
\usepackage{booktabs}       
\usepackage{xcolor}         
\usepackage{subcaption}
\usepackage{amsfonts}
\usepackage{amsthm}
\usepackage{amsmath}
\usepackage{tikz}
\usepackage{mathtools}
\usepackage{nicefrac}       
\usepackage{microtype}      
\usepackage[export]{adjustbox}
\usepackage{bm}
\usepackage{csquotes}

\def\ci{\perp\!\!\!\!\!\perp}
\newtheorem{proposition}{Proposition}

\title{
	A Simple Unified Approach to Testing High-Dimensional Conditional Independences
	for Categorical and Ordinal Data
}
\author{
	Ankur Ankan\textsuperscript{\rm 1}, Johannes Textor\textsuperscript{\rm 1}
}
\affiliations{
    \textsuperscript{\rm 1}Data Science, Institute for Computing and Information Sciences, Radboud University, Nijmegen, The Netherlands \\
    \{ankur.ankan, johannes.textor\}@ru.nl

}

\usepackage{bibentry}

\begin{document}

\maketitle

\begin{abstract}
	Conditional independence (CI) tests underlie many approaches to model
	testing and structure learning in causal inference. Most existing CI
	tests for categorical and ordinal data stratify the sample by the
	conditioning variables, perform simple independence tests in each
	stratum, and combine the results.  Unfortunately, the statistical power
	of this approach degrades rapidly as the number of conditioning
	variables increases.  Here we propose a simple unified CI test for
	ordinal and categorical data that maintains reasonable calibration and
	power in high dimensions. We show that our test
	outperforms existing baselines in model testing and structure learning
	for dense directed graphical models while being comparable for sparse
	models. Our approach could be attractive for
	causal model testing because it is easy to implement, can be used
	with non-parametric or parametric probability models, has the symmetry
	property, and has reasonable computational requirements.
\end{abstract}

\section{Introduction}

Scientific claims should be falsifiable.
In causal inference, falsifiable claims can be read off graphical
causal models using graphical criteria. Many types of 
graphical models including DAGs, MAGs, undirected 
graphical models, and chain graphs imply conditional 
independences (CIs). Therefore, statistical CI tests play an 
important role in model testing and structure learning -- which itself 
can be seen as a sequence of iterative model tests and post-hoc modifications
performed by an algorithm.

Unfortunately, compared to simple (unconditional) independence testing, CI
testing is much harder; for example, a non-parametric CI test for continuous data that
is both calibrated and has power 
does not exist \citep{Bergsma2004,shah2020hardness}.
Some assumptions therefore need to be imposed on the relationships between the involved variables. A large 
amount of work has been done on quantifying conditional dependence
using measures such as mutual information \citep{cover1999elements}, the
Hilbert-Schmidt independence criterion \citep{gretton2005measuring}, and
distance covariance \citep{szekely2007measuring}; see also
\citet{josse2014measures} for an overview. A wide variety of CI tests
has been proposed based on concepts such as ranks \citep{weihs2018symmetric}, kernel
methods \citep{pfister2016kernel}, copulas \citep{kojadinovic2009tests}, knock-off
sampling \citep{watson2019testing}, nearest neighbors \citep{berrett2019nonparametric},
and generalized covariance measures \citep{shah2020hardness}. 

Due to the prominent role of the structure learning problem in the causal inference
literature, CI tests are often developed and evaluated with structure learning 
in mind. In applied literature, however, structure 
learning is not yet widely used and graphical causal models are often
constructed by hand. For example, a recent review of the use of DAGs in health research
\citep{Tennant2019} found hundreds of papers in which DAGs were constructed, mainly to inform covariate adjustment strategies.
Perhaps surprisingly, none of these 
DAG models was tested against the dataset it was supposed to represent, posing a severe
risk for inferences based on these models -- it seems unlikely that researchers
can come up with a correct graphical structure based on their 
intuition and domain knowledge alone. 

We suspect that perceived or real issues with existing CI tests are part of the reason that
DAG model testing and structure learning aren't more widely used.
We argue that a CI test for practical use should have the following properties: 
\begin{enumerate}
\item it should be \emph{simple} in the sense that it is based on elementary statistical concepts that most researchers are familiar with;
\item it should be \emph{symmetric} -- tests of $X \ci Y \mid Z$ and $Y \ci X \mid Z$
should deliver the same result;
\item it should be \emph{computationally efficient} since even for hand-constructed models,
it can be necessary to perform hundreds of tests (at least one per missing edge); and 
\item it should have reasonable calibration and power in real-world data.
\end{enumerate}

For continuous data, one could argue that all these conditions are fulfilled by a simple
test where we perform two regressions (not necessarily linear ones)
 $E[X \mid Z]$ and $E[Y \mid Z]$, and 
determine the correlation between their residuals, which should be 0 under CI 
if our regressions accurately model the conditional expectation \citep{thoemmes18_pm}.
Unfortunately, many important datasets do not consist of only continuous variables.
CI testing for categorical and ordinal data has received considerably less attention in the
 literature, perhaps because from a theoretical point of view, it appears to be 
a much simpler problem: CI testing for such data can be done by essentially
stratifying the sample according to $Z$, performing separate CI tests in each stratum,
and combining the results
(see also Remark 4 in \citet{shah2020hardness}). Since there are only finitely many strata,
such tests can be non-parametric, calibrated, and have power against meaningful alternatives
at the same time.  

\begin{figure}
\begin{subfigure}{.5\textwidth}
\centering
\begin{tikzpicture}
\tikzstyle{every node}=[align=center, inner sep=1pt]
\node at (0,0) {Income \\ \small(Incm)};
\node at (2,0) {Workclass \\ \small(Wrkc)};
\node at (4,0) {Education \\ \small(Edct)};
\node (mrts) at (6,0) {Marital~Status \\ \small(MrtS)};

\node at (0,-1) {Occupation \\ \small(Occp)};
\node (rltn) at (4,-1) {Relationship \\ \small(Rltn)};
\node at (2,-1) {Race};
\node (sex) at (6,-1) {Sex};

\node (hrpw) at (0,-2) {Hours~Per~Week \\ \small(HrPW)};
\node at (4,-2) {Native~Country \\ \small(NtvC)};
\node (age) at (2,-2) {Age};
\draw [very thick] (sex) -- (rltn);
\draw [very thick] (mrts) -- (rltn);
\draw [very thick] (hrpw) -- (age);
\end{tikzpicture}
\caption{}
\end{subfigure}
\begin{subfigure}{.5\textwidth}
\centering
\begin{tabular}{llrr}
\multicolumn{4}{c}{Z=Age, Sex} \\
$X$ & $Y$ & p & df \\
\hline
Edct & Wrkc & 1.00 & 1050 \\
Occp & Wrkc & 1.00 & 840 \\
Rltn & HrPW &  .99 & 210 \\
Incm & Occp & .08 & 168 \\
Incm & Wrkc &  .50 & 70 \\
Incm & HrPW & .003 & 42 \\
\hline
\end{tabular}
\caption{}
\end{subfigure}
\caption{(a) Skeleton estimated by the stable PC algorithm \citep{Colombo2014} from 1000 samples of 
the adult income data using the default 
CI test in the R package `bnlearn', a stratified mutual information test.
Almost no variables are connected even though there are substantial pairwise relationships
between most variables in the data.
(b) A closer inspection of test results reveals high degrees of 
freedom that sometimes exceed the sample size. Such tests are strongly biased towards
independence because very little information is used per stratum.}
\label{fig:pcerror}
\end{figure}

To our knowledge, there currently exists no CI test for discrete and ordinal
data that meets the above criteria. 
For illustration,
consider the widely known ``adult income'' or ``US census income'' data \citep{Dua2019,kohavi1996} that 
contains rich categorical variables such as ``Native Country'' (41 categories),
 ``Education'' (16), and ``Occupation`` (14),
along with ordinal variables such as ``HoursPerWeek'' 
and ``Income'' (binarized at cutoff \$50K/year).
Like for many sociological datasets, most pairs of variables
are substantially but not strongly dependent, and these dependences are not easily
``explained away'' by conditioning on other variables. In other words, although there is
no ``true structure'' known for this dataset, any reasonable structure should be dense.
Yet, structure learning based on stratification-based CI tests 
typically returns very sparse graphs 
(Figure~\ref{fig:pcerror}a). This happens because as low-dimensional CI tests rightly fail to identify
any independences, higher dimensions will be considered and at some point the tests will
become unreliable (Figure~\ref{fig:pcerror}b). Thus, for high-dimensional
data, the mutual information test and related tests such as chi-square and $G^2$
fulfill our first 3 desiderata but not the 4th.

In our experience, such issues are not pathological edge cases, but
are routinely obtained when applying ``default'' constraint-based structure learning algorithms 
to real-world datasets containing discrete variables (as many do); 
indeed, this is a frequent source of confusion and frustration
for first-time users or students trying to get acquainted with causal inference methodology. 

This paper proposes a simple CI testing approach for categorical and
ordinal data that fulfills our desiderata and outperforms calibration 
and power of state-of-the-art methods for
high-dimensional conditioning sets. Our approach combines a residual for 
ordinal data \citep{Li2012} with a multidimensional location test, Hotelling's $T^2$ test.
We can use any suitable estimator of conditional probabilities; here, we 
show results using logistic regression and random forests. 

%
%
%

\section{Background \& related work}

We consider one-dimensional discrete or ordinal variables $X, Y$ 
and a possibly multi-dimensional discrete or ordinal 
variable $Z = Z_1,\ldots,Z_k$ with joint probability 
density $p(x,y, z)$, and write $\mathbf{x}=(x_1,\ldots,x_n)$ for a sample from 
$X$ of size $n$. 
We write the expectation of a variable $X$ as $\mathbb{E}[X]$, conditional
expectations as $\mathbb{E}[X \mid Z]$,
the covariance between $X$ and $Y$ as $\textbf{cov}(X,Y)$, 
the variance of $X$ as $\textbf{var}(X)$,
and the covariance estimated from samples $\mathbf{x},\mathbf{y}$ as
$\textbf{cov}(\mathbf{x},\mathbf{y})$.
We say that $X$ and $Y$ are conditionally independent given $Z$, or $X \ci Y \mid Z$,
if  for all $z$ with $p(z)>0$, $p( x,y \mid Z=z) =
 p(x \mid Z=z )p(y \mid Z=z )$ \citep{Dawid1979}.

We can roughly categorize CI tests into three groups.  First,
\emph{stratification tests} split the data into subsets according to $Z$,
perform a marginal independence test $X \ci Y$ within each subset, and combine
the results.  This approach is natural for discrete $Z$, but can also be
applied to continuous $Z$ upon binning. Such tests are relatively simple and
usually symmetric by construction, but they rapidly lose power when $Z$ becomes
high-dimensional, even if irrelevant variables are added to $Z$. Some such
tests like chi-square also lose validity altogether for smaller datasets with
high-dimensional $Z$ because they are based on asymptotic statistics and
stratification can lead to only a few samples available for each individual
marginal test. This issue can be addressed by using exact tests instead
\citep{Tsamardinos2010}, improving calibration but not necessarily power.

Second, \emph{variable importance tests} compare a probability model
$\hat{p}( x \mid y, z )$ to a simpler model $\hat{p}( x \mid z )$ based on some
goodness-of-fit metric. If the simpler model does not fit substantially 
worse, one accepts the claim $X \ci Y \mid Z$. This approach is attractive because
it can leverage any statistical model with a reasonable goodness of fit metric
or nested model test; e.g., we could perform such 
a test for binary data simply by fitting logistic regressions $X \sim Y + Z_1 + \ldots + Z_k$ and 
examining the coefficient of $Y$ and its sampling error. 
A downside of this approach is its inherent asymmetry:  depending on the probability model used,
a test of $X \ci Y \mid Z$ could yield a different result than a test of 
$Y \ci X \mid Z$, which could be confusing because CI is a symmetric property. 

Third, \emph{residualization tests} fit two models 
$\mathbb{E}[ X \mid Z ]$ and $\mathbb{E}[ Y \mid Z ]$, and examine the relationships 
between the residuals $R_{x_i} = x_i - \mathbb{E}[ X \mid Z = z_i ]$ 
and $R_{y_i} = y_i - \mathbb{E}[ Y \mid Z = z_i ]$.
The validity of these tests rests on a theorem by 
\citet{Daudin1980}, which implies
that when $X \ci Y \mid Z$ and residuals are valid
($\mathbb{E}[R_{X}] = \mathbb{E}[R_{Y}] = 0$), then $\mathbb{E}[R_{X} R_{Y}]=0$.
Therefore, we can  test CI by examining a multiplicative association measure
between $R_{X}$ and $R_{Y}$,
which should be 0 under CI.
Such measures include correlation or the generalized covariance
measure \citep{shah2020hardness}. This approach has the attractive feature that 
it is symmetric by construction. Instead, we can also conduct
CI tests by attempting to predict $R_{X}$ from $R_{Y}$ or vice versa
\citep{Shah2017,HeinzeDeml2018}; such tests are not necessarily symmetric. 

Most existing CI tests for categorical data are based on stratification. 
This includes chi-square and $G^2$/mutual information based tests such as 
those implemented in the R packages `bnlearn' \citep{Scutari2014} and 
`pcalg' \citep{Kalisch2012}. \citet{Tsamardinos2010} show how the calibration and 
power of such tests can be improved by using exact versions or their Monte Carlo
approximations. More recently, \citet{MarxV19} proposed 
a variable importance test called \emph{SCCI} that uses an approximation
to Kolmogorov complexity. We will use SCCI as a modern baseline for comparison, although
that comparison is not always straightforward because SCCI only provides a 
pseudo p-value without calibration guarantees. We are not aware of existing dedicated 
residualization tests for categorical or ordinal data -- 
for example, at present, none of the tests implemented in the R package 
`CondIndTests' \citep{HeinzeDeml2018}
will run on fully categorical data where every variable has more than 2 levels.
We could of course leverage existing residualization tests by dummy-coding all
categorical and binary data, performing multiple comparisons, and somehow combining
the results; however, this procedure would result in information loss for ordinal data,
and it is not necessarily obvious how the individual results would have to be
combined to maintain calibration under the null and to obtain meaningful effect sizes.
We therefore leave this comparison for future work. 

\section{Test development}

Here we propose a CI test for categorical and ordinal data that is based on the residualization
approach. The main issue with developing such a test is that there is no straightforward
definition of a residual for categorical or ordinal data, since subtraction is 
meaningless for such variables. Throughout we consider a CI test between 
an ordinal or categorical variable $X$ with $k$ levels, an ordinal or categorical variable
$Y$ with $r$ levels, and a set of ordinal or categorical conditional variables $Z$.

\subsection{Residual}

We will use a uniform residual for all tests. For an observation $y$ of an ordinal (possibly binary),
we use the residual for ordinal data by \citet{Li2012}. Given 
a sample $\mathbf{y}$ of $Y$ and an estimate $\hat{p}(y)$ of the distribution 
$p(y)$, this Li-Shepherd-residual (LS-residual) is defined as
$$
R_{ y_i } = \hat{p}( Y < y_i ) - \hat{p}( Y > y_i ) \ .
$$
Although LS-residuals generally do not (and cannot) have the observed-minus-expected (OME) form
that is typically associated with a residual, 
they do share important properties with OME residuals \citet{Li2012}.
The exception is binary $Y \in \{0,1\}$, in which the LS-residual 
does have the OME form and reduces to the standard OLS residual
for binomial variables, i.e., 
$$
R_{ y_i } = y_i - \hat{p}( Y = 1 )\ .
$$
Similarly, the conditional residual for samples $(\mathbf{y},\mathbf{z})$
is defined as 
$$
R_{ y_i \mid z_i } = \hat{p}( Y < y_i | Z = z_i ) - 
	\hat{p}( Y > y_i | Z = z_i ) 
$$

\subsection{Test statistic}

We now define test statistics for each possible combination of ordinal 
and categorical variables that we can encounter. These test statistics
are closely related to each other. In each case, we assume 
that residuals are formed with respect to the test in question; for example, 
if we test $X \ci Y \mid Z$, then $R_{\mathbf{x}}$ is based on 
$\hat{p}( x \mid \mathbf{z} )$. Here we will define the test statistics
and derive their asymptotic distributions; throughout, our proofs are adapted/generalized
versions of the proof in \citet{lishepherd2010}, which is based on M-estimation
theory and the delta method. Therefore, our asymptotic results require the assumption 
that an M-estimator is used to estimate the conditional probabilities 
$\hat{p}( x \mid \mathbf{z})$ and $\hat{p}( y \mid \mathbf{z} )$. 

We begin with the simplest case. If both variables are  
ordinal, we use the following test statistic, 
which is the squared generalized covariance measure (GCM) \cite{shah2020hardness}:

$$ Q_1(\mathbf{x},\mathbf{y}) = \frac{1}{n} 
	\frac{ \left( R_{\mathbf{x}} \cdot R_{\mathbf{y}} \right)^2 }{ \textbf{var}( R_{\mathbf{x}}  R_{\mathbf{y}} ) } \ . $$ 
\begin{proposition}
\label{prop:q1}
If $X \ci Y \mid Z$, then asymptotically $Q_1(\mathbf{x},\mathbf{y}) \sim \chi^2(1)$. 
\end{proposition}
\begin{proof}
	\citet{shah2020hardness} prove that the non-squared version of $Q_1$ is asymptotically standard normal.
	However, here we show this instead by slightly adapting the proof in~\citet{lishepherd2010}, which is
	simpler and generalizable to higher dimensions in an intuitive and straightforward manner. 
	See Appendix for details. One important difference is that the proof by~\citet{shah2020hardness} 
 	is based on an assumption that the estimator $\hat{p}$ converges ``quickly enough'', whereas ours is 
	based on M-estimation theory. 
\end{proof}

Next, consider categorical $X$ with $k$ indexed categories and ordinal 
$Y$. For the sample $\mathbf{x}$ we define the binary 
\emph{indicator} variables (also known as ``dummy variables'') $\mathbb{I}(x_i=j), 1 \leq j \leq k$ where
$\mathbb{I}(x_i=j)=1$ if $x_i=j$ and $\mathbb{I}(x_i=j)=0$ otherwise.
We now consider all dot products between the ordinal residuals $R_\mathbf{y}$ and the residuals for the first $k-1$ dummy variables of $X$, 
$$
d = (R_{\mathbb{I}(\mathbf{x}=1)} \cdot R_{\mathbf{y}}, \, \ldots \ ,
R_{\mathbb{I}(\mathbf{x}=k-1)} \cdot R_{\mathbf{y}}
)
$$
and use it to define our test statistic analogously to a Hotelling's test:
$$
Q_2(\mathbf{x},\mathbf{y}) = \frac{1}{n} \left( d \times \hat{\Sigma}_d^{-1} \times d^T \right) \ ,
$$
where the matrix $\hat{\Sigma}_d$ contains the estimated covariances between the components of $ (R_{\mathbb{I}(\mathbf{x}=1)} \odot R_{\mathbf{y}}, \, \ldots \ , R_{\mathbb{I}(\mathbf{x}=k-1)} \odot R_{\mathbf{y}})$ (here, $\odot$ denotes the element-wise product). Note that we drop one of the dummy variables (without information loss) because otherwise, $\hat{\Sigma}_d$ would not be full rank. 
\begin{proposition}
\label{prop:q2}
If $X \ci Y \mid Z$, then asymptotically $Q_2(\mathbf{x},\mathbf{y}) \sim \chi^2(k-1)$. 
\end{proposition}
\begin{proof}
	A multidemensional analogue of Proposition~\ref{prop:q1}. See Appendix for details.
\end{proof}
Finally, consider categorical $X$ and $Y$ with 
$k>1$ and $r>1$ categories, respectively. Then we define our vector $d$ as the pairwise dot products between the residuals for the indicator variables of $X$ and $Y$
\begin{eqnarray*}
d &  =  & (R_{\mathbb{I}(\mathbf{x}=1)} \cdot R_{\mathbb{I}(\mathbf{y}=1)}, \, \ldots \ ,
R_{\mathbb{I}(\mathbf{x}=k-1)} R_{\mathbb{I}(\mathbf{y}=1)}, \, \ldots \, ,
\\
 &  & R_{\mathbb{I}(\mathbf{x}=1)} \cdot R_{\mathbb{I}(\mathbf{y}=r-1)}, \, \ldots \ ,
R_{\mathbb{I}(\mathbf{x}=k-1)} R_{\mathbb{I}(\mathbf{y}=r-1)}
)
\end{eqnarray*}
and define our test statistic in the same way as for the previous case:
$$
Q_3(\mathbf{x},\mathbf{y}) = \frac{1}{n} \left( d \times \hat{\Sigma}_d^{-1} \times d^T \right)
$$
Analogously to the previous case, we then obtain 
\begin{proposition}
\label{prop:categorical}
If $X \ci Y \mid Z$, then asymptotically $Q_3(\mathbf{x},\mathbf{y})\sim \chi^2((k-1)(r-1))$.
\end{proposition}
\begin{proof}
	Very similar to Proposition~\ref{prop:q2}. See Appendix for details.
\end{proof}

\subsection{Conditional probability model}

The above simple combination of LS residuals with a Hotelling's test provides
us with a generic framework for CI testing for categorical and ordinal data. To
conduct such tests, we need to choose an estimator of the 
involved conditional probabilities.  Ideally, this should be a statistical
model that is able to naturally incorporate both ordinal and categorical
variables, and provides a simple way to compute the LS residuals. In this
paper, we consider two estimators: 1) Generalized Linear Model (GLM), and 2)
Random Forest with probability prediction \citep{Malley2012}. We chose GLM because it 
is an M-estimator and therefore covered by our proofs in the previous section.
The random forest is not an M-estimator but we hypothesized that it might nevertheless
work well in practice and could be good at discarding irrelevant information from high-dimensional
conditioning sets, which we hoped would benefit the power and robustness of the
resulting CI test. 

\subsection{Relationship to the Partial Copula approach}

\citet{petersenhansen2021} use partial copulas to construct a CI test for continuous data. We note that this approach is closely related to ours. 
For continuous $Y$, the \emph{partial copula} of $Y$ given $Z$ is defined as 

$$
 C_{y_i \mid z_i} = \hat{p}(Y \leq y_i \mid Z=z_i )
$$

Therefore, 
$$
C_{y_i \mid z_i} = \frac{1}{2} ( ( \hat{p}(Y \leq y_i \mid Z=z_i ) - \hat{p}(Y > y_i \mid Z=z_i ) )  + 1 )
$$

where the difference $\hat{p}(Y \leq y_i \mid Z=z_i ) - \hat{p}(Y > y_i \mid
Z=z_i ) $ is similar to the LS residual. Specifically, in the LS residual, the
left term is $\hat{p}(Y < y_i \mid Z=z_i )$ rather than $\hat{p}(Y \leq y_i
\mid Z=z_i )$. In a certain sense, the partial copula could be seen as a
``limit'' of LS residuals: Consider a continuous variable $Y$ defined on some
interval $[a,b]$, and an ordinal version $\hat{Y}^{(n)}$ generated from $Y$ by
binning using $n$ equidistant cutoffs. Then as $n \to \infty$, the LS residual
$R_{\hat{y}^{(n)}_i \mid z_i}$ converges to $2 C_{y_i^{(n)} \mid z_i} - 1$. 

\section{Empirical analysis}
We now show empirical results comparing our method to some of the other
state-of-the-art CI tests. We compare our Generalized Linear Models based test
(GLM) and Random Forest based test (RFT) to $ 3 $ other tests: 1) Mutual
Information based test (MI) \citep{edwards2012introduction}, 2) Monte Carlo
Permutation test (MC-MI) \citep{edwards2012introduction}, and 3) SCCI
\citep{MarxV19}. For ordinal data, we also compare it to the
Jonckheere-Terpstra test (JT) \citep{jttest}. We use the implementation of MI,
MC-MI, and JT from the R package `bnlearn' (ver. 4.7) \citep{Scutari2014}, and
SCCI from the R package `SCCI' (ver. 1.2) \citep{MarxV19}. For GLM, we use a
multinomial logistic regression (binomial logistic regression for binary data)
from the R package `nnet' (ver. 7.3.17) \citep{mass2002} to compute the
prediction probabilities required for computing residuals. In the case of
ordinal data, we use a proportional odds logistic regression model from the R
package `VGAM' (ver. 1.1.7) \citep{vgam2022} that takes the order of the
categories into account. For RFT, we use the implementation of probability
forests \citep{Malley2012} from the R package `ranger' (ver. 0.13.1)
\citep{rangerwright2017} to compute prediction probabilities. We use the
default hyperparameter values except reducing the number of trees to $ 50 $ to
reduce computational cost with no loss in performance. All analyses
were run on an Intel i5-10600k CPU with 32 GBs of RAM.

\subsection{Calibration}
\label{sec:calibration}

\begin{figure}
	\centering
	\includegraphics{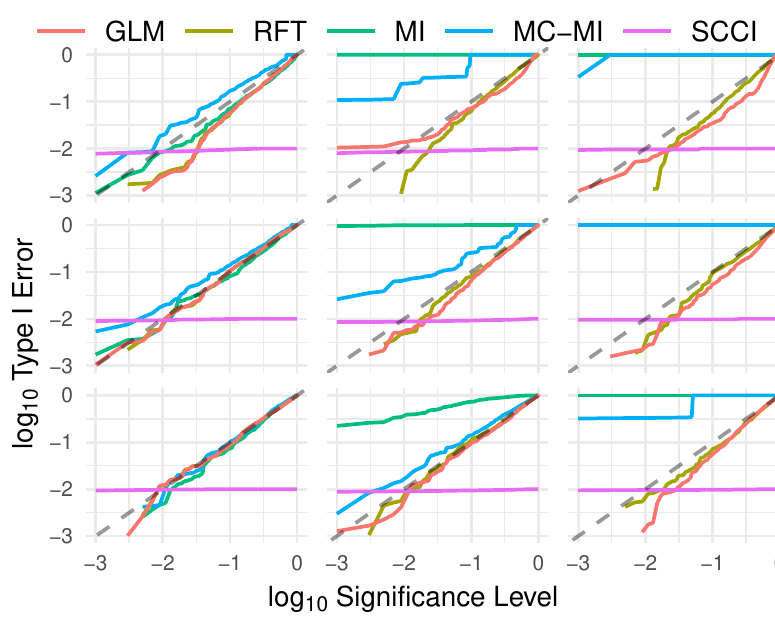}
	\caption{Type I error vs significance level for sample sizes (top to
	bottom): $ [20, 40, 80] $ and number of conditional variables (left to
	right): $ [1, 3, 5] $ on conditionally independent simulated binary
	datasets.}
	\label{fig:calibration}
\end{figure}

To determine calibration, we analyzed the Type I error rate of the tests
at varying significance levels. Under the null $ X \ci Y \mid Z $, 
for a perfectly calibrated test, we expect the p-value to be uniformly
distributed, hence a plot of significance level versus fraction of rejected
null hypotheses should be a straight diagonal line.
For this analysis, we generated $ 500 $ datasets with all binary variables
satisfying the null $ X \ci Y \mid Z $ according to the following structure:
\begin{center}
\begin{tikzpicture}[yscale=.75]
\node (x) at (0,1) {$X$};
\node (y) at (2,1) {$Y$};
\node (z) at (1,1) {$Z_1$};
\node at (3,1) {$Z_2$};
\node at (4,1) {$\ldots$};
\node at (5,1) {$Z_k$};
\draw [->] (z) -- (y);
\draw [->] (z) -- (x);
\end{tikzpicture}
\end{center}
We started by generating uniformly
random binary samples $\mathbf{z_i} $ for the conditional variables.  Then we
sampled $\mathbf{x}$ and $\mathbf{y}$ from the binomial distribution $ B(2,
\mathbf{z_1}/3) $. Finally, we computed 500 p-values by 
testing $ X \ci Y \mid Z $ on each generated dataset using all the
tests while varying the number of conditional variables and sample sizes.

Figure~\ref{fig:calibration} shows the results of our analysis on a 
log-log scale that emphasizes the values in the common range for $p$-value
cutoffs. GLM and RFT are better calibrated in most cases except when the number of
conditional variables is low with a relatively high sample size (bottom left plot
in Figure~\ref{fig:calibration}), where MC-MI is better calibrated.  Especially
for high-dimensional CI tests in small samples, GLM and RFT 
are much better calibrated compared to the other tests. SCCI is not calibrated
at all; this is because it only gives pseudo p-values which are all around $ 0.01 $,
with values greater than $ 0.01 $ representing independence.

\subsection{Discrimination}

\begin{figure}
	\centering
	\includegraphics{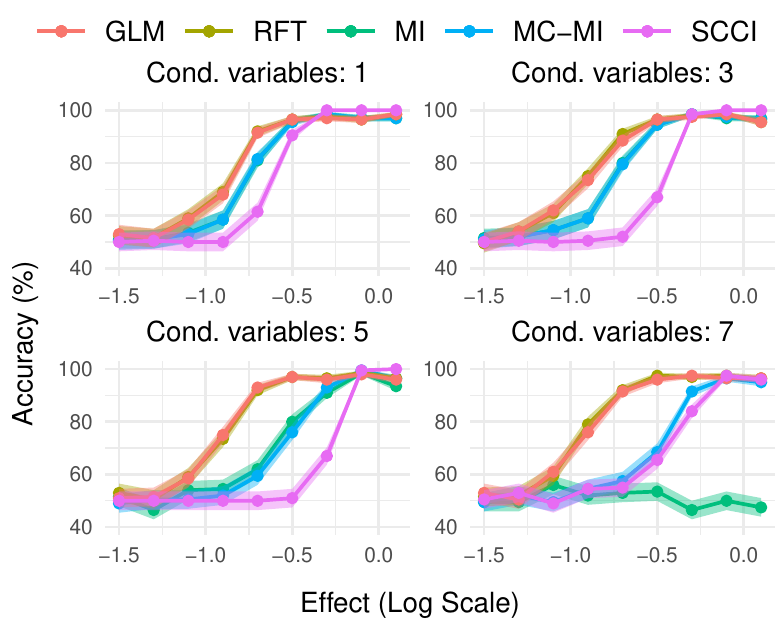}
	\caption{Accuracy (shading: mean $\pm$ standard error, $N=200$) of classifying
	simulated binary datasets (sample size: $1000$) as conditionally
	dependent or independent.}
	\label{fig:cat_discrimination}
\end{figure}

Having addressed calibration under null, we now show a
discrimination analysis to compare the accuracy of tests on correctly accepting or
rejecting the null. We conducted this analysis on both categorical and ordinal data.
For deciding between dependence and independence
we used a $p$-value threshold of $ 0.05 $
for all tests except SCCI for which we use its designated threshold of $0.01$.

\subsubsection{Categorical data.}
Our analysis is similar to the one performed by \citet{Tsamardinos2010}. 
We generated data according to the following general DAG structure:
\begin{center}
\begin{tikzpicture}[yscale=.75]
\node (x) at (0,0) {$X$};
\node (y) at (2,0) {$Y$};
\node (z) at (1,1) {$Z_1$};
\node at (2,1) {$Z_2$};
\node at (3,1) {$\ldots$};
\node at (4,1) {$Z_k$};
\draw [->] (x) edge node [midway, above] {?} (y);
\draw [->] (z) -- (y);
\draw [->] (z) -- (x);
\end{tikzpicture}
\end{center}
The $Z_{\geq 2}$ act as irrelevant ``nuisance variables''. Our task is 
to determine whether the edge $X \to Y$ is present (dependent) or absent
(independent). We generated binary data using the logistic model
\begin{equation}
p( y_i = 1 ) = \lambda ( \sum_{X \text{ is a parent of } Y} \beta x_i ) 
\label{eqn:logisticmodel}
\end{equation}
where $\lambda(x) = e^x/(e^x+1)$ is the logistic function and $\beta$ is 
the ``effect'' (which we fixed to the same value for all edges).
Varying the effect and the number of $k$ of conditional variables,
we simulated 100 dependent and 100 independent datasets consisting of 1000 samples
for each combination.

Figure~\ref{fig:cat_discrimination} shows the accuracy of classifying the
simulated datasets. All tests perform poorly for tiny effects and strongly
for huge effects, but we find that GLM and RFT outperforms the other
tests in the ``switch regime'' in between, with the difference becoming more
pronounced when adding nuisance variables. Thus, our test appears to be more 
robust to noise. 

\subsubsection{Ordinal data.}

\begin{figure}
	\centering
	\includegraphics{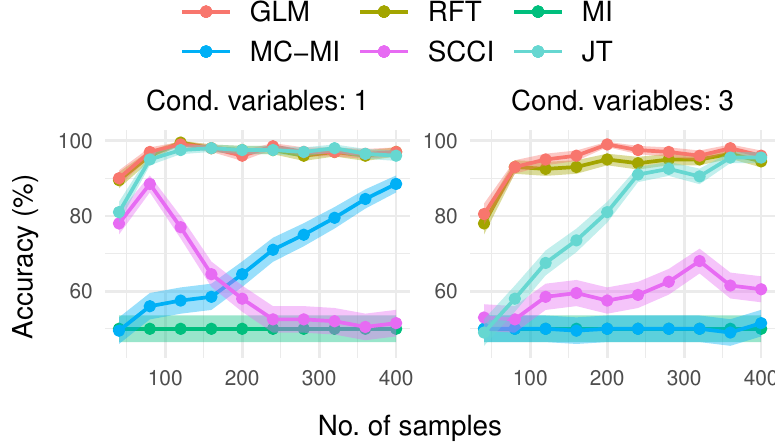}
	\caption{Accuracy (shading: mean $\pm$ standard error, N=200) of
		classifying simulated ordinal data (8 levels per variable) as
		conditionally dependent or independent.}
	\label{fig:accuracy_ord}
\end{figure}

We next simulated ordinal data from the same DAG structure as follows: we first
generated samples $\mathbf{z_i}, 1 \leq i \leq k$ from the binomial
distribution $B(8, 0.5) $. To generate independent data, we independently
sampled $\mathbf{x}$ and $\mathbf{y}$ from $ B(8,\frac{\mathbf{z_1}}{9})$.
To generate dependent data, we then randomly permuted $\mathbf{z_1}$.
Figure~\ref{fig:accuracy_ord} shows the accuracy of the tests computed on
$ 100 $ conditionally dependent and independent datasets. In this setup, we
varied the sample size rather than an effect size. For $k=1$, GLM, RFT, and JT
performed equally well and better than the other tests, which do not take the
order of the categories into account.  But for $k=3$, GLM and RFT were more
accurate than JT in small samples.

\section{Applications}
We evaluated our test on two important applications of CI tests: (1) model testing and (2)
structure learning. We used the same baselines as in the previous section
for comparison.

\subsection{Model testing}
\label{section:model_testing}

\begin{figure}
	\centering
	\includegraphics{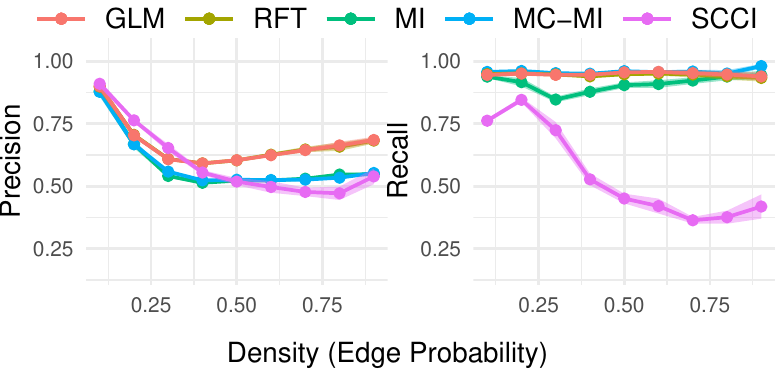}
	\caption{Precision and recall of testing implied versus non-implied CIs
		 in binary data (N=1000) simulated from random DAGs on $ 20 $ variables.
		 Shading: mean $\pm$ standard error.} 
	\label{fig:model_testing}
\end{figure}

The CIs implied by a DAG should hold in the dataset(s) it is supposed to represent.
Therefore, we can scrutinize a DAG by testing implied CIs.
In our analysis, we simulated datasets from randomly generated DAGs and
compared how well the tests can correctly detect the implied CIs of the DAG in
the simulated dataset. We started by generating random DAGs on $20$
variables. We connected each pair of variables at a fixed probability,
with all edges oriented according to a pre-defined topological ordering. 
We then simulated binary datasets with $ 1000 $ samples using our logistic
model (Equation~\ref{eqn:logisticmodel}) setting $\beta=0.15$.
Then we used the CI tests to test one implied CI per missing edge 
and an equal number of randomly generated CIs in the dataset.
For generating a random CI $ X \ci Y
\mid Z $, we first selected $ X $ and $ Y $ variables randomly
and then selected a random number of conditional variables $ Z $ from
the remaining variables. Using d-separation, we determined which 
randomly generated CIs truly hold in the DAG. 
All CIs were then
tested in the simulated data and precision and recall were computed
(Figure~\ref{fig:model_testing}). The precision of all methods was comparable in 
sparse DAGs, as the implied CIs had relatively few
conditional variables. But in denser DAGs, GLM and RFT had better precision.
Recall was comparable for all
tests except SCCI, which did not perform well.

\subsection{Structure learning}

CI tests also play an important role in constraint-based structure learning,
where algorithms iteratively perform CI tests to determine whether two
variables in the model are connected by an edge or not. For learning the
network structures in this section, we use the implementation of fast
``stable'' variant \citep{Colombo2014} of PC algorithm
\citep{spirtes2000causation} from the R package 'pcalg` \citep{Kalisch2012}.

\subsubsection{Simulated data.}
\begin{figure}
	\centering
	\includegraphics{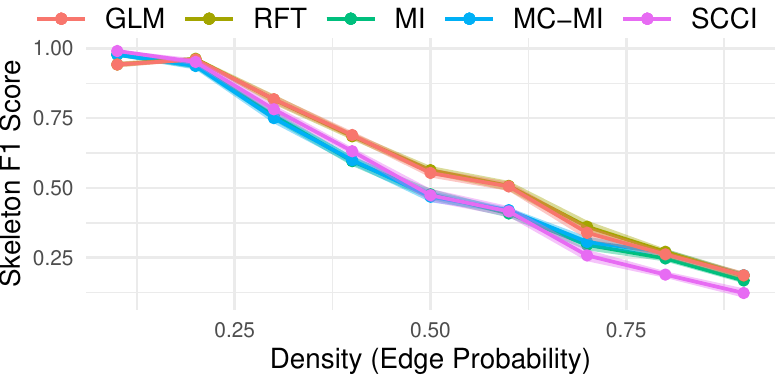}
	\caption{Structure learning on simulated data: Mean F1 scores (10
		 simulated binary datasets per point) for varying graph densities. Each
		 dataset contains 1000 samples and is simulated from a randomly
		 generated DAG with 20 variables. Shading: mean $\pm$ standard error.}
	\label{fig:sl_density}
\end{figure}

We first show empirical results of structure learning on simulated
datasets.  We randomly generated DAGs with varying densities as
in the previous section. For each DAG, we generated $1000$ 
samples using our logistic model (Equation~\ref{eqn:logisticmodel}) using $\beta=0.15$. 
We used the PC algorithm to learn the
model skeleton for each simulated dataset and compared it to the
true skeleton using the F1~score (Figure~\ref{fig:sl_density}).
All tests performed comparably for sparse DAGs. For denser
DAGs, the stratification tests MI and MC-MI performed the worst, whereas
the variable importance test SCCI performed better. Yet GLM and RFT
substantially outperformed SCCI, as expected
given the results in Figure~\ref{fig:cat_discrimination}. 

\subsubsection{Synthetic benchmark data.}
\begin{figure}[h]
	\centering
	\includegraphics{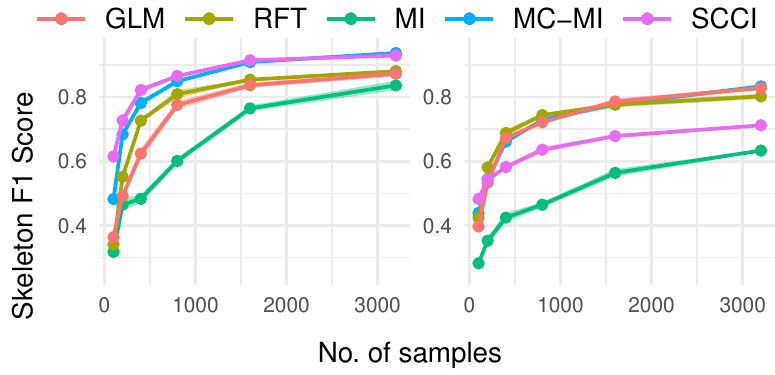}
	\caption{Structure learning on datasets ``alarm'' (left) and ``insurance'' (right):
		Mean F1~scores (10 subsampled datasets per sample size) of the learned
		model skeletons.  Presence of an edge is considered the ``positive'' case
		for F1~scores. Shading: mean $\pm$ standard error.}
	\label{fig:sl}
\end{figure}
We next evaluated the performance of the tests on two commonly used datasets in
structure learning benchmarks, the ``alarm'' \citep{beinlich1989alarm} and
``insurance'' \citep{binder1997adaptive} datasets, which again are simulated
data from known ground truth.  We used the PC algorithm to learn the skeleton
using subsampled datasets of varying sizes, and determined F1 scores
(Figure~\ref{fig:sl}).  SCCI and MC-MI perform best for the alarm model whereas
GLM, RFT, and MC-MI perform equally well for the insurance model (except for
very low sample size). Importantly, both these models are very sparse. The
alarm model has 37 variables and 46 edges, hence the edge probability is $
\frac{46}{(37*36)/2} = 0.069 $. The insurance model is slightly denser with 27
variables and 52 edges and an edge probability of $ \frac{52}{(27*26)/2} = 0.15
$.  As PC algorithm removes most of the edges in early iterations (i.e. low
conditioning variables) for sparse models, when tests perform poorly in these
initial iterations, it can lead to a cascading effect where the algorithm ends
up doing many more tests and may reach a higher number of conditioning
variables. We saw this happening with GLM and RFT in these datasets as it is
slightly less well calibrated than MC-MI for low number of conditional
variables and high sample size (Figure~\ref{fig:calibration}). Moreover,
MC-MI's bias towards classifying a CI as independent
(Figure~\ref{fig:calibration}) helps in learning sparser models.

\subsubsection{Real data.}

\begin{figure}[h]
	\centering
	\begin{subfigure}{0.6\columnwidth}
		\centering
		\begin{tikzpicture}[scale=1.77]
			\tikzstyle{every node}=[inner sep=1pt, align=center]
			\scriptsize
			\node (hrpw) at (0:1.2cm) {HrPW};
			\node (race) at (-33:1.2cm) {Race};
			\node (ntvc) at (-61:1.2cm) {NtvC};
			\node (edct) at (-98:1.2cm) {Edct};
			\node (age) at (-131:1.2cm) {Age};
			\node (mrts) at (-164:1.2cm) {MrtS};
			\node (rltn) at (-196:1.2cm) {Rltn};
			\node (sex) at (-225:1.2cm) {Sex};
			\node (occp) at (-262:1.2cm) {Occp};
			\node (incm) at (-300:1.2cm) {Incm};
			\node (wrkc) at (-333:1.2cm) {Wrkc};

			\draw [] (hrpw) -- (race) -- (ntvc) -- (edct) -- (age) -- 
				(mrts) -- (rltn) -- (sex) -- (occp) -- (incm) -- (wrkc) -- (hrpw);

			\draw [] (hrpw) -- (incm) -- (age) -- (hrpw);

			\draw [] (wrkc) -- (edct) -- (occp) -- (wrkc);

			\draw [] (edct) -- (rltn) -- (age);

			\draw [] (hrpw) -- (edct) -- (incm);
		\end{tikzpicture}
		\caption{}
		\label{fig:sl_adult_model}
	\end{subfigure}%
	\begin{subfigure}{0.4\columnwidth}
		\includegraphics[scale=0.85]{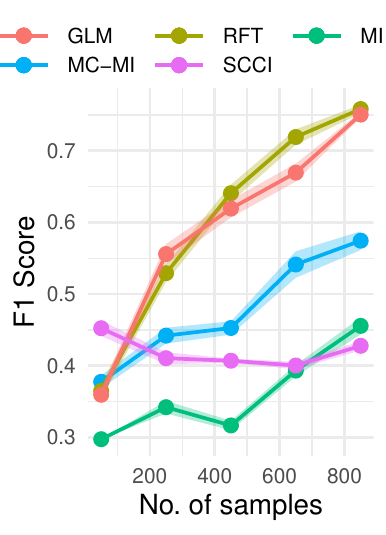}
		\caption{}
		\label{fig:sl_adult}
	\end{subfigure}
	\caption{Structure learning on adult income data. (a) Skeleton
		estimated by the stable PC algorithm from the data in
		Figure~\ref{fig:pcerror} when using our Random-Forest based
		test (RFT). (b) Mean F1 score ($10$ adult income data subsamples
		per point)
		when comparing $d$-connected variable pairs in the CPDAG to
		correlated variable pairs in the dataset. Presence of
		d-connection is used as the positive case for the
		F1 score. Shading: mean $\pm$ standard
		error.}
\end{figure}
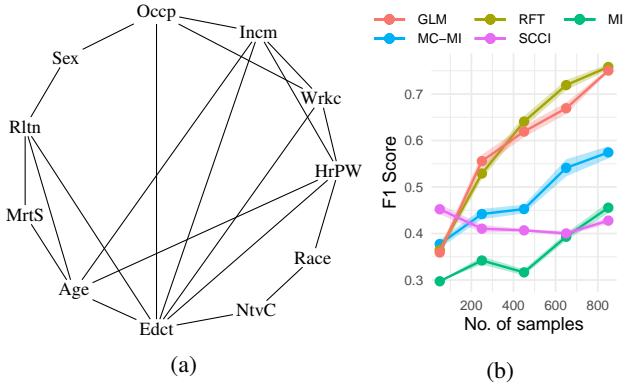

Finally, we return to the adult income data.
Using PC structure learning with our RFT, a 
more connected skeleton was generated (Figure~\ref{fig:sl_adult_model}) compared
to the earlier baseline (Figure~\ref{fig:pcerror}).
For systematic quantitative 
evaluation, we discretized the variable ``Age'' into the categories $ <21 $,
$ 21 \text{-} 30 $, $\ldots$, $ 61 \text{-} 70 $, $ >70 $ and the variable 
``HoursPerWeek'' into the categories
$ <=20 $, $ 21 \text{-} 30 $, $ 30 \text{-} 40 $, $ >40 $. 
We then tested whether pairwise dependence in the data
corresponded to d-connectedness in learned structures --
a reasonable requirement that is evaluable even in the absence of a ground 
truth structure.
To determine pairwise dependences, we
performed chi-square independence tests for each variable pair and considered 
the variables dependent if the 
root mean square error of approximation, a chi-square effect size defined by 
$\sqrt{(\chi^2-\text{df})/(n \text{df})}$, is greater than $ 0.05 $. 
We then learned model structures on subsamples of the
dataset, and calculated F1 scores by comparing d-connected variables in
the learned CPDAG to dependent variables in the dataset. 
GLM and RFT performed best except for the smallest sample sizes
(Figure~\ref{fig:sl_adult}).

Taken together, our results show that our GLM and RFT perform similarly or 
slightly worse than baselines for low-dimensional, limited data, 
but equally or better in the other cases. Especially in our motivating scenario
of high-dimensional, tightly correlated datasets, the
performance gain is substantial. 
GLM and RFT were slower than most baselines but 
faster than MC-MI (Figure~\ref{fig:runtime}).

\begin{figure}
	\centering
	\includegraphics{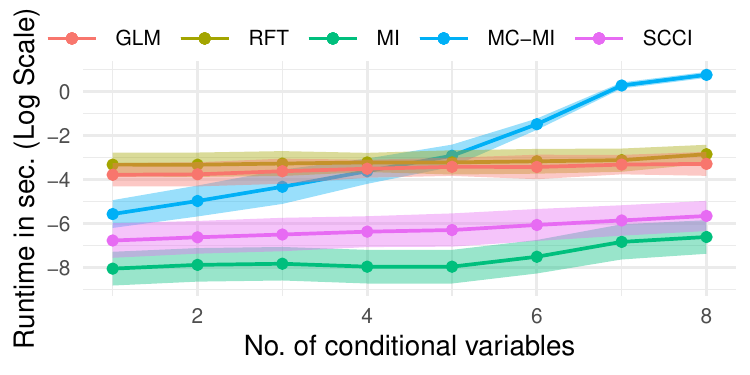}
	\vspace{-2.5mm}
	\caption{Mean runtime (100 CI tests per point) with varying numbers of
		conditional variables and $1000$ samples per dataset; data is generated
		like in Figure~\ref{fig:cat_discrimination}. Shading: mean $\pm$ standard
		error.}
	\label{fig:runtime}
\end{figure}

\section{Conclusion}

We have proposed a residualization-based approach for CI testing for discrete
and ordinal data. We think this approach could be especially attractive for
manual model testing in empirical research because  (1) it is symmetric by
construction; (2) it is based on rather elementary statistical concepts (here
we used Hotelling's test and GLM/random forest); and (3) its computational cost
is reasonable.  In addition to these qualitative advantages, we showed that it
compares favorably to existing alternatives with respect to calibration,
discrimination, and power, and is useful in the context of structure learning
when the networks can be expected to be dense, which is the case for many
real-world datasets.

Although the test is sensitive to model misspecification when computing the
residuals, it provides the flexibility to choose a parametric M-estimator that 
is suitable for the data at hand. Alternatively, we found the non-parametric 
Random Forest estimator to perform well empirically.

We have shown that the LS residuals that our approach is based on are closely
related to partial copulas. We therefore believe that it should be possible to
combine the CI testing approach proposed by \citet{petersenhansen2021} with our
approach to obtain a single, unifying CI testing framework for any kind of
mixed dataset. 

Since our approach outperforms baselines in high- but not low-dimensional
settings, structure learning algorithms might be able to get the ``best of both
worlds'' by adaptively choosing our test or a simpler one based on some
estimation of how well the simpler test should perform. However, appropriate
criteria for switching between the two tests would need to be developed first.

For now, we hope that the combination of our residualization approach and a
random forest might be a reasonably robust ``plug-in'' solution for causal
model testing and structure learning in datasets containing ordinal and
categorical variables. We hope that this might help to persuade more empirical
researchers to test their graphical causal models or to try out structure
learning algorithms.

\section{Acknowledgments}
The authors would like to thank Tom Heskes for discussions. 
This work was done in part while the authors were
visiting the Simons Institute for the Theory of Computing.
\bibliography{bibliography}

\begin{thebibliography}{36}
\providecommand{\natexlab}[1]{#1}

\bibitem[{Beinlich et~al.(1989)Beinlich, Suermondt, Chavez, and
  Cooper}]{beinlich1989alarm}
Beinlich, I.~A.; Suermondt, H.~J.; Chavez, R.~M.; and Cooper, G.~F. 1989.
\newblock The ALARM Monitoring System: A Case Study with two Probabilistic
  Inference Techniques for Belief Networks.
\newblock In Hunter, J.; Cookson, J.; and Wyatt, J., eds., \emph{AIME 89},
  247--256. Springer.
\newblock ISBN 978-3-642-93437-7.

\bibitem[{Bergsma(2004)}]{Bergsma2004}
Bergsma, W.~P. 2004.
\newblock Testing conditional independence for continuous random variables.
\newblock Technical Report 2004-049, EURANDOM.

\bibitem[{Berrett and Samworth(2019)}]{berrett2019nonparametric}
Berrett, T.~B.; and Samworth, R.~J. 2019.
\newblock Nonparametric independence testing via mutual information.
\newblock \emph{Biometrika}, 106(3): 547--566.

\bibitem[{Binder et~al.(1997)Binder, Koller, Russell, and
  Kanazawa}]{binder1997adaptive}
Binder, J.; Koller, D.; Russell, S.; and Kanazawa, K. 1997.
\newblock Adaptive probabilistic networks with hidden variables.
\newblock \emph{Machine Learning}, 29(2): 213--244.

\bibitem[{Colombo and Maathuis(2014)}]{Colombo2014}
Colombo, D.; and Maathuis, M.~H. 2014.
\newblock Order-Independent Constraint-Based Causal Structure Learning.
\newblock \emph{Journal of Machine Learning Research}, 15: 3921--3962.

\bibitem[{Cover(1999)}]{cover1999elements}
Cover, T.~M. 1999.
\newblock \emph{Elements of information theory}.
\newblock John Wiley \& Sons.
\newblock ISBN 978-0-471-24195-9.

\bibitem[{Daudin(1980)}]{Daudin1980}
Daudin, J.~J. 1980.
\newblock Partial association measures and an application to qualitative
  regression.
\newblock \emph{Biometrika}, 67(3): 581--590.

\bibitem[{Dawid(1979)}]{Dawid1979}
Dawid, A.~P. 1979.
\newblock Conditional Independence in Statistical Theory.
\newblock \emph{Journal of the Royal Statistical Society. Series B
  (Methodological)}, 41(1): 1--31.

\bibitem[{Dua and Graff(2017)}]{Dua2019}
Dua, D.; and Graff, C. 2017.
\newblock {UCI} Machine Learning Repository.

\bibitem[{Edwards(2012)}]{edwards2012introduction}
Edwards, D. 2012.
\newblock \emph{Introduction to graphical modelling}.
\newblock Springer.
\newblock ISBN 978-0-387-95054-9.

\bibitem[{Gretton et~al.(2005)Gretton, Bousquet, Smola, and
  Sch{\"o}lkopf}]{gretton2005measuring}
Gretton, A.; Bousquet, O.; Smola, A.; and Sch{\"o}lkopf, B. 2005.
\newblock Measuring Statistical Dependence with Hilbert-Schmidt Norms.
\newblock In Jain, S.; Simon, H.~U.; and Tomita, E., eds., \emph{Algorithmic
  Learning Theory}, 63--77. Springer.
\newblock ISBN 978-3-540-31696-1.

\bibitem[{Heinze-Deml, Peters, and Meinshausen(2018)}]{HeinzeDeml2018}
Heinze-Deml, C.; Peters, J.; and Meinshausen, N. 2018.
\newblock Invariant Causal Prediction for Nonlinear Models.
\newblock \emph{Journal of Causal Inference}, 6(2).

\bibitem[{Jonckheere(1954)}]{jttest}
Jonckheere, A.~R. 1954.
\newblock A Distribution-Free k-Sample Test Against Ordered Alternatives.
\newblock \emph{Biometrika}, 41(1/2): 133--145.

\bibitem[{Josse and Holmes(2014)}]{josse2014measures}
Josse, J.; and Holmes, S. 2014.
\newblock Measures of dependence between random vectors and tests of
  independence. Literature review.
\newblock arXiv:1307.7383.

\bibitem[{Kalisch et~al.(2012)Kalisch, M\"achler, Colombo, Maathuis, and
  B\"uhlmann}]{Kalisch2012}
Kalisch, M.; M\"achler, M.; Colombo, D.; Maathuis, M.~H.; and B\"uhlmann, P.
  2012.
\newblock Causal Inference Using Graphical Models with the {R} Package {pcalg}.
\newblock \emph{Journal of Statistical Software}, 47(11): 1--26.

\bibitem[{Kohavi(1996)}]{kohavi1996}
Kohavi, R. 1996.
\newblock Scaling up the Accuracy of Naive-Bayes Classifiers: A Decision-Tree
  Hybrid.
\newblock In \emph{Proceedings of the Second International Conference on
  Knowledge Discovery and Data Mining}, KDD'96, 202–207. AAAI Press.

\bibitem[{Kojadinovic and Holmes(2009)}]{kojadinovic2009tests}
Kojadinovic, I.; and Holmes, M. 2009.
\newblock Tests of independence among continuous random vectors based on
  Cram{\'e}r--von Mises functionals of the empirical copula process.
\newblock \emph{Journal of Multivariate Analysis}, 100(6): 1137--1154.

\bibitem[{Li and Shepherd(2010)}]{lishepherd2010}
Li, C.; and Shepherd, B.~E. 2010.
\newblock Test of Association Between Two Ordinal Variables While Adjusting for
  Covariates.
\newblock \emph{Journal of the American Statistical Association}, 105(490):
  612--620.

\bibitem[{Li and Shepherd(2012)}]{Li2012}
Li, C.; and Shepherd, B.~E. 2012.
\newblock A new residual for ordinal outcomes.
\newblock \emph{Biometrika}, 99(2): 473--480.

\bibitem[{Malley et~al.(2012)Malley, Kruppa, Dasgupta, Malley, and
  Ziegler}]{Malley2012}
Malley, J.~D.; Kruppa, J.; Dasgupta, A.; Malley, K.~G.; and Ziegler, A. 2012.
\newblock Probability Machines.
\newblock \emph{Methods of Information in Medicine}, 51(01): 74--81.

\bibitem[{Marx and Vreeken(2019)}]{MarxV19}
Marx, A.; and Vreeken, J. 2019.
\newblock Testing Conditional Independence on Discrete Data using Stochastic
  Complexity.
\newblock volume~89, 496--505. PMLR.

\bibitem[{Petersen and Hansen(2021)}]{petersenhansen2021}
Petersen, L.; and Hansen, N.~R. 2021.
\newblock Testing Conditional Independence via Quantile Regression Based
  Partial Copulas.
\newblock \emph{Journal of Machine Learning Research}, 22(70): 1--47.

\bibitem[{Pfister et~al.(2018)Pfister, Bühlmann, Schölkopf, and
  Peters}]{pfister2016kernel}
Pfister, N.; Bühlmann, P.; Schölkopf, B.; and Peters, J. 2018.
\newblock Kernel-based tests for joint independence.
\newblock \emph{Journal of the Royal Statistical Society: Series B (Statistical
  Methodology)}, 80(1): 5--31.

\bibitem[{Scutari and Denis(2014)}]{Scutari2014}
Scutari, M.; and Denis, J.-B. 2014.
\newblock \emph{Bayesian Networks with Examples in {R}}.
\newblock Boca Raton: Chapman and Hall.
\newblock ISBN 978-1-4822-2558-7, 978-1-4822-2560-0.

\bibitem[{Shah and B\"{u}hlmann(2017)}]{Shah2017}
Shah, R.~D.; and B\"{u}hlmann, P. 2017.
\newblock Goodness-of-fit tests for high dimensional linear models.
\newblock \emph{Journal of the Royal Statistical Society: Series B (Statistical
  Methodology)}, 80(1): 113--135.

\bibitem[{Shah, Peters et~al.(2020)}]{shah2020hardness}
Shah, R.~D.; Peters, J.; et~al. 2020.
\newblock The hardness of conditional independence testing and the generalised
  covariance measure.
\newblock \emph{Annals of Statistics}, 48(3): 1514--1538.

\bibitem[{Spirtes et~al.(2000)Spirtes, Glymour, Scheines, and
  Heckerman}]{spirtes2000causation}
Spirtes, P.; Glymour, C.~N.; Scheines, R.; and Heckerman, D. 2000.
\newblock \emph{Causation, prediction, and search}.
\newblock MIT press.
\newblock ISBN 9780262194402.

\bibitem[{Sz{\'e}kely et~al.(2007)Sz{\'e}kely, Rizzo, Bakirov
  et~al.}]{szekely2007measuring}
Sz{\'e}kely, G.~J.; Rizzo, M.~L.; Bakirov, N.~K.; et~al. 2007.
\newblock Measuring and testing dependence by correlation of distances.
\newblock \emph{The annals of statistics}, 35(6): 2769--2794.

\bibitem[{Tennant et~al.(2020)Tennant, Murray, Arnold, Berrie, Fox, Gadd,
  Harrison, Keeble, Ranker, Textor, Tomova, Gilthorpe, and
  Ellison}]{Tennant2019}
Tennant, P. W.~G.; Murray, E.~J.; Arnold, K.~F.; Berrie, L.; Fox, M.~P.; Gadd,
  S.~C.; Harrison, W.~J.; Keeble, C.; Ranker, L.~R.; Textor, J.; Tomova, G.~D.;
  Gilthorpe, M.~S.; and Ellison, G. T.~H. 2020.
\newblock {Use of directed acyclic graphs (DAGs) to identify confounders in
  applied health research: review and recommendations}.
\newblock \emph{International Journal of Epidemiology}, 50(2): 620--632.

\bibitem[{Thoemmes, Rosseel, and Textor(2018)}]{thoemmes18_pm}
Thoemmes, F.; Rosseel, Y.; and Textor, J. 2018.
\newblock Local fit evaluation of structural equation models using graphical
  criteria.
\newblock \emph{Psychological Methods}, 23(1): 27--41.

\bibitem[{Tsamardinos and Borboudakis(2010)}]{Tsamardinos2010}
Tsamardinos, I.; and Borboudakis, G. 2010.
\newblock Permutation Testing Improves Bayesian Network Learning.
\newblock In \emph{Machine Learning and Knowledge Discovery in Databases},
  322--337. Springer.
\newblock ISBN 978-3-642-15939-8.

\bibitem[{Venables and Ripley(2002)}]{mass2002}
Venables, W.~N.; and Ripley, B.~D. 2002.
\newblock \emph{Modern Applied Statistics with S}.
\newblock New York: Springer, fourth edition.
\newblock ISBN 0-387-95457-0.

\bibitem[{Watson and Wright(2021)}]{watson2019testing}
Watson, D.~S.; and Wright, M.~N. 2021.
\newblock Testing conditional independence in supervised learning algorithms.
\newblock \emph{Machine Learning}, 110(8): 2107--2129.

\bibitem[{Weihs, Drton, and Meinshausen(2018)}]{weihs2018symmetric}
Weihs, L.; Drton, M.; and Meinshausen, N. 2018.
\newblock {Symmetric rank covariances: a generalized framework for
  nonparametric measures of dependence}.
\newblock \emph{Biometrika}, 105(3): 547--562.

\bibitem[{Wright and Ziegler(2017)}]{rangerwright2017}
Wright, M.~N.; and Ziegler, A. 2017.
\newblock ranger: A Fast Implementation of Random Forests for High Dimensional
  Data in C++ and R.
\newblock \emph{Journal of Statistical Software}, 77(1): 1–17.

\bibitem[{Yee(2022)}]{vgam2022}
Yee, T.~W. 2022.
\newblock \emph{{VGAM}: Vector Generalized Linear and Additive Models}.
\newblock R package version 1.1-7.

\end{thebibliography}

\clearpage
\newpage

\appendix
\onecolumn
\section{Proofs of propositions}
For proving the asymptotic distributions of our test statistics under the null,
we use the asymptotic distribution proof from \citet{lishepherd2010} as a
template and define m-estimators \citep{mestimation} for our test statistics.
Given a vector of parameters $ \bm{\theta} $ and a function $
\Psi_i(\bm{\theta}) = \Psi(Y_i, X_i, \bm{Z}_i; \bm{\theta}) $, such that the
following conditions are satisfied:

\begin{enumerate}
	\item $ \bm{\theta} $ can be estimated using $ \sum_{i=1}^n \Psi_i(\bm{\theta}) = 0 $
	\item $ \Psi_i $ doesn't depend on $ i $ or $ n $.
	\item $ \mathbb{E}[\Psi_i(\bm{\theta})] = 0 $
\end{enumerate}

From m-estimation theory, if $ \Psi $ is suitably smooth, then as $ n \rightarrow \infty $,

\begin{equation}
	\sqrt{n}(\hat{\bm{\theta}} - \bm{\theta}) \rightarrow \mathcal{N}(0, V(\bm{\theta}))
\end{equation}

where $ V(\bm{\theta}) = A(\bm{\theta})^{-1} B(\bm{\theta}) [A(\bm{\theta})^{-1}]' $, $ A(\bm{\theta}) = \mathbb{E} \left[ - \frac{\partial}{\partial \bm{\theta}} \Psi_i (\bm{\theta}) \right] $, and $ B(\bm{\theta}) = \mathbb{E}[\Psi_i(\bm{\theta}) \Psi_i{\bm{\theta}}]' $.

Further, from the delta method, given a function $ g(\bm{\theta}) $ which is a smooth function of $ \bm{\theta} $, we have:

\begin{equation}
\label{eq:asymptotic_normal}
	\sqrt{n} \left[ g(\hat{\bm{\theta}}) - g(\bm{\theta}) \right] \rightarrow \mathcal{N}(0, \sigma^2)
\end{equation}

where $ \sigma^2 = \left[ \frac{\partial}{\partial \bm{\theta}} g(\bm{\theta}) \right] V(\bm{\theta}) \left[ \frac{\partial}{\partial \bm{\theta}} g(\bm{\theta}) \right]'$ 

We now define $ \bm{\theta} $, $ \Psi_i $, and $ g(\bm{\theta}) $ for each of
our three test statistics. We separate our parameter vector into three subvectors as
$ \bm{\theta} = \{ \bm{\theta}^{\bm{X}}, \bm{\theta}^{\bm{Y}},
\bm{\theta}^{\bm{T}} \} $, where $ \bm{\theta}^{\bm{X}} $ and $ \bm{\theta}^{\bm{Y}}
$ are the parameters of the models $ P(X | \bm{Z} ) $ and $ P(Y | \bm{Z})
$, respectively. Since these parameters and the m-estimator over these parameters 
are the same for all the test statistics, we first define a partial m-estimator, $ \Phi_i $ on
these parameters.

\begin{equation}
\label{eq:likelihood_est}
\Phi_i(\bm{\theta}) = \begin{cases}
	\frac{\partial}{\partial \bm{\theta}^{\bm{Y}}} l_Y (Y_i, \bm{Z}_i; \bm{\theta^{Y}}) \\
	\frac{\partial}{\partial \bm{\theta}^{\bm{X}}} l_X (X_i, \bm{Z}_i; \bm{\theta^{X}}) \\
	\end{cases}
\end{equation}

where $ l_X $ and $ l_Y $ are the log-likelihood functions of the probability
models $ P(Y | \bm{Z}) $ and $ P(X | \bm{Z}) $ (for example, these could be multinomial models) 
with parameters $ \bm{\theta}^{Y} $ and $ \bm{\theta}^{X} $ respectively.  In the
following sub-sections, we use $ \Phi_i(\bm{\theta}) $ as a partial m-estimator
and define the remaining parameters ($ \bm{\theta}^{\bm{T}} $), the m-estimator
on the remaining parameters ($ \Psi_i(\bm{\theta}^{\bm{T}}) $), and a function
($ g(\bm{\theta}) $) for each of the test statistics.

\subsection{Proof for proposition 1}
Using the partial m-estimator defined in \ref{eq:likelihood_est}, we now define the remaining parameters $ \bm{\theta}^T $, m-estimator function on $ \bm{\theta}^T $, and $ g(\bm{\theta}) $ for $Q_1$.

\begin{equation}
	\begin{split}
		\bm{\theta}^T &= (\theta_1, \theta_2) \\
		\Psi (X_i, Y_i, \mathbf{Z}_i; \bm{\theta}^T) &= 
			\begin{cases}
				R_{y_i} R_{x_i} - \theta_1 \\
				(R_{y_i} R_{x_i})^2 - \theta_2
			\end{cases} \\
		g(\bm{\theta}) &= \frac{\theta_1}{\sqrt{\theta_2 - \theta_1^2}} \\
	\end{split}
\end{equation}

Here, we have chosen $ g(\bm{\theta}) $ such that $ Q_1 = (\sqrt{n}
g(\bm{\theta}))^2 $. Further, solving the equation $ \sum_i \Psi(X_i, Y_i,
\mathbf{Z}_i; \bm{\theta}^T) = 0 $ gives us the estimates for $ \theta_1 $ and
$ \theta_2 $,

\begin{equation}
	\begin{split}
		\hat{\theta}_1 &= \mathbb{E}[R_{\mathbf{x}} R_{\mathbf{y}}] \\
		\hat{\theta}_2 &= \mathbb{E}[(R_{\mathbf{x}} R_{\mathbf{y}})^2] \\
	\end{split}
\end{equation}

As both $ \Psi $ and $ g(\bm{\theta}) $ are smooth functions, from m-estimation theory and the delta method as $ n \rightarrow \infty $,

$$ \sqrt{n} (g(\hat{\bm{\theta}}) - g(\bm{\theta})) \rightarrow \mathcal{N}(0, \sigma^2) $$

Since under the null $ g(\bm{\theta}) = 0 $,

\begin{equation}
		\sqrt{n}g(\hat{\bm{\theta}}) \rightarrow \mathcal{N}(0, \sigma^2) \\
\end{equation}

where $ \sigma^2 $ can be computed as defined in Eq~\ref{eq:asymptotic_normal}.

\begin{equation}
	\begin{split}
		A(\bm{\theta}) &= E \left[ -\frac{\partial}{\partial \theta_1} \Psi_i(\bm{\theta}), -\frac{\partial}{\partial \theta_2} \Psi_i(\bm{\theta}) \right] = E \begin{bmatrix} 1 & 0 \\ 0 & 1 
				\end{bmatrix} = \mathbb{I}_2 \\
		B(\bm{\theta}) &= \mathbb{E} \begin{bmatrix}
			(R_{y_i} R_{x_i} - \theta_1)(R_{y_i} R_{x_i} - \theta_1) & (R_{y_i} R_{x_i} - \theta_1) ((R_{y_i} R_{x_i})^2 - \theta_2) \\
			((R_{y_i} R_{x_i})^2 - \theta_2) (R_{y_i} R_{x_i} - \theta_1) &  ((R_{y_i} R_{x_i})^2 - \theta_2) ((R_{y_i} R_{x_i})^2 - \theta_2)  \\
				\end{bmatrix} \\
			       &= \begin{bmatrix}
				       \theta_2 & \mathbb{E}[(R_{\mathbf{x}} R_{\mathbf{y}})^3] \\
				       \mathbb{E}[(R_{\mathbf{x}} R_{\mathbf{y}})^3 & \mathbb{E}[(R_{\mathbf{x}} R_{\mathbf{y}})^4] - \theta_2^2 \\
			       \end{bmatrix} \text{(as $\theta_1 = 0 $ under the null)} \\
		V(\bm{\theta}) &= A(\bm{\theta})^{-1} B(\bm{\theta}) [A(\bm{\theta})^{-1}]' = B(\bm{\theta}) \\
		\frac{\partial}{\partial \theta} g(\bm{\theta}) &=  
					\begin{bmatrix} 
						\frac{\partial}{\partial \theta_1} g(\bm{\theta}) &
						\frac{\partial}{\partial \theta_2} g(\bm{\theta})
					\end{bmatrix} 
			= \begin{bmatrix}
				\theta_2^{-\frac{1}{2}} &
				0 \\
			   \end{bmatrix} \\
		\sigma^2 &= \begin{bmatrix} \theta_2^{-\frac{1}{2}} & 0 \end{bmatrix} 
			\begin{bmatrix} \theta_2 & \mathbb{E}[(R_{\mathbf{x}} R_{\mathbf{y}})^3] \\ 
				\mathbb{E}[(R_{\mathbf{x}} R_{\mathbf{y}})^3] & \mathbb{E}[(R_{\mathbf{x}} R_{\mathbf{y}})^4] - \theta_2^2 \\ 
				\end{bmatrix} 
				\begin{bmatrix} \theta_2^{-\frac{1}{2}} \\ 0 \\ \end{bmatrix} \\
		        &= ((\theta_2)^{-\frac{1}{2}})^2 \theta_2 + 0 + 0 = 1 \\
	\end{split}
\end{equation}

Hence, $ \sqrt{n}g(\bm{\theta}) $ has an asymptotic standard normal
distribution. Since, $ Q_1 = (\sqrt{n}g(\bm{\theta}))^2 $, $ Q_1 $ is
asymptotically chi-square distributed with $ 1 $ degree of freedom.

\subsection{Proof for proposition 2}
Similar to the last proof, we start by defining the remaining parameters, the m-estimator function, and $ g(\bm{\theta}) $ for $ Q_2 $.

\begin{equation}
	\begin{split}
		& \bm{\theta}^T = (\bm{\theta_1}, \bm{\theta_2}) \\
		& \bm{\theta_1} = (\theta_1^1, \theta_1^2, \cdots, \theta_1^{k-1}) \\
		& \bm{\theta_2} = (\theta_2^{11}, \cdots, \theta_2^{1(k-1)}, \theta_2^{21}, \cdots, \theta_2^{2(k-1)}, \cdots, \theta_2^{(k-1) (k-1)}) \\
		& \Psi(X_i, Y_i, \mathbf{Z}_i; \bm{\theta}^T) = 
			\begin{cases}
				R_{{\mathbb{I}(x=j)}_i}.R_{y_i} - \theta_1^j \, \, \forall j \in \{ 1, \cdots, k-1 \} \\
				(R_{{\mathbb{I}(x=j)}_i}.R_{y_i} \times R_{{\mathbb{I}(x=k)}_i}.R_{y_i}) - \theta_{2}^{jk} \,\, \forall j, k \in \{ 1, \cdots, k-1 \} \\
			\end{cases} \\
		& g(\bm{\theta}) = \bm{\theta_1} \bm{\wedge}
	\end{split}
\end{equation}
where $ \bm{\wedge} $ is defined as:
\begin{equation}
	\begin{split}
		\bm{\wedge} &= \bm{\Sigma}^{-\frac{1}{2}} \\
		\Sigma^{i,j} &= \mathbb{E}[(R_{\mathbb{I}(\mathbf{x}=i)}.R_\mathbf{y})(R_{\mathbb{I}(\mathbf{x}=j)}.R_\mathbf{y})] - \mathbb{E}[R_{\mathbb{I}(\mathbf{x}=i)}.R_\mathbf{y}] \mathbb{E}[R_{\mathbb{I}(\mathbf{x}=j)}.R_\mathbf{y}] \\
			    &= \theta_2^{ij} - \theta_1^{i} \theta_1^{j} \\
	\end{split}
\end{equation}

Here, we have chosen $ g(\bm{\theta}) $ such that $ Q_2 = \sum (\sqrt{n} g(\bm{\theta}))^2 $. Further, solving the equation $ \sum_i \Psi(X_i, Y_i, \mathbf{Z}_i; \bm{\theta}^T) = 0 $ gives us estimates for $ \bm{\theta_1} $ and $ \bm{\theta_2} $:
\begin{equation}
	\begin{split}
		\hat{\theta}_{1}^j &= \mathbb{E}[R_{\mathbb{I}(\mathbf{x}=j)} R_\mathbf{y}] \\
		\hat{\theta}_2^{jk} &= \mathbb{E}[(R_{\mathbb{I}(\mathbf{x}=j)} R_\mathbf{y}) (R_{\mathbb{I}(\mathbf{x}=k)} R_\mathbf{y})] \\
	\end{split}
\end{equation}

Similar to the last case as both $ \Psi $ and $ g(\bm{\theta}) $ are smooth functions with $ g(\bm{\theta}) = 0 $ under the null,

$$ \sqrt{n} g(\hat{\bm{\theta}}) \rightarrow \mathcal{N}(\bm{0}, \bm{\sigma}^2) $$

We again use Eq~\ref{eq:asymptotic_normal} to compute $ \bm{\sigma} $ with $ A(\bm{\theta}) $ defined as,

\begin{equation}
	A(\bm{\theta}) = \mathbb{E} \left[ - \frac{\partial}{\partial \theta_1^{j}} \Psi_i(\bm{\theta}), - \frac{\partial}{\partial \theta_2^{jk}} \Psi_i(\bm{\theta}) \right] = \mathbb{I}_{k(k-1)}
\end{equation}

The matrix $ B(\bm{\theta}) $ has $ 4 $ types of terms. To simplify the notation, we split $ B(\bm{\theta}) $ into sub-matrics $ \bm{B_{00}} $, $ \bm{B_{01}} $, $ \bm{B_{10}} $, and $ \bm{B_{11}} $ with shapes 
$ (k-1) \times (k-1) $, $ (k-1) \times (k-1)^2 $, $ (k-1)^2 \times (k-1) $, and $ (k-1)^2 \times (k-1)^2 $ respectively:

\begin{equation}
	\begin{split}
	B(\bm{\theta}) &= \begin{bmatrix}
		\bm{B_{00}} & \bm{B_{01}} \\
		\bm{B_{10}} & \bm{B_{11}} \\
			\end{bmatrix} \\
	B_{00}^{i, j} &= \mathbb{E}[(R_{\mathbb{I}(x=i)} R_y - \theta_1^i)(R_{\mathbb{I}(x=j)} R_y - \theta_1^j)] \\
		       &= \theta_2^{ij} - 2 \theta_1^i \theta_1^j + \theta_1^i \theta_1^j = \theta_2^{ij} = \bm{\Sigma} \, \, \text{(as $ \bm{\theta_1} = 0 $ under the null)} \\
		B_{10}^{i, jk} &= \mathbb{E}[(R_{\mathbb{I}(x=i)} R_y - \theta_1^i)(R_{\mathbb{I}(x=j)} R_y \times R_{\mathbb{I}(x=k)} R_y - \theta_2^{jk})] \\
		       &= \mathbb{E}[R_{\mathbb{I}(x=i)} R_y R_{\mathbb{I}(x=j)} R_y R_{\mathbb{I}(x=k)} R_y] - \theta_1^i \theta_2^{jk} - \theta_1^i \theta_2^{jk} - \theta_1^i \theta_2^{jk} \\
		       &= \mathbb{E}[R_{\mathbb{I}(x=i)} R_y R_{\mathbb{I}(x=j)} R_y R_{\mathbb{I}(x=k)} R_y] \, \, \text{(as $ \bm{\theta_1} = 0 $ under the null)} \\
		\bm{B_{01}} &= \bm{B_{10}}^T \\
		B_{11}^{ij, kl} &= \mathbb{E}[(R_{\mathbb{I}(x=i)} R_y R_{\mathbb{I}(x=j)} R_y - \theta_2^{ij}) (R_{\mathbb{I}(x=k)} R_y R_{\mathbb{I}(x=l)} R_y - \theta_2^{kl})] \\
		       &= \mathbb{E}[(R_{\mathbb{I}(x=i)} R_y R_{\mathbb{I}(x=j)} R_y R_{\mathbb{I}(x=k)} R_y R_{\mathbb{I}(x=l)} R_y)] - \theta_2^{ij} \theta_2^{kl} - \theta_2^{ij} \theta_2^{kl} + \theta_2^{ij} \theta_2^{kl} \\
		       &= \mathbb{E}[(R_{\mathbb{I}(x=i)} R_y R_{\mathbb{I}(x=j)} R_y R_{\mathbb{I}(x=k)} R_y R_{\mathbb{I}(x=l)} R_y)] - \theta_2^{ij} \theta_2^{kl} 
	\end{split}
\end{equation}

Using these terms, we can now compute $ \bm{\sigma} $.
\begin{equation}
	\begin{split}
		\frac{\partial}{\partial \theta} g(\bm{\theta}) &= \begin{bmatrix} 
			\frac{\partial}{\partial \bm{\theta_1}} g(\bm{\theta}) & \frac{\partial}{\partial \bm{\theta_2}} g(\bm{\theta}) \\
								   \end{bmatrix} \\
		\left[ \frac{\partial}{\partial \bm{\theta_1}} g(\bm{\theta}) \right]^{j,l} &= \sum_{p=1}^{k-1} \theta_1^p \frac{\partial}{\partial \theta_1^j} \wedge^{p,j} + \wedge^{l,j} = \wedge^{l,j} \\
		\left[ \frac{\partial}{\partial \bm{\theta_2}} g(\bm{\theta}) \right]^{j, lm} &= \sum_{p=1}^{k-1} \theta_1^p \frac{\partial}{\partial \theta_2^{lm}} \wedge^{p,j} = 0\\
			\frac{\partial}{\partial \bm{\theta}} g(\bm{\theta}) &= \begin{bmatrix} \bm{\wedge} & \bm{0} \end{bmatrix}\\
			\bm{\sigma}^2 &= \bm{B_{00}} (\bm{\wedge})^2 = \bm{\Sigma} (\bm{\Sigma}^{-1/2})^2 = \mathbb{I}_{(k-1)}
	\end{split}
\end{equation}

Hence, $ \sqrt{n} g(\bm{\theta}) $ has an asymptotic $ k-1 $ variate standard normal distribution. Since, $ Q_2 = \sum (\sqrt{n} g(\bm{\theta}))^2 $, $ Q_2 $ is asymptotically chi-square distributed with $ k-1 $ degrees of freedom.

\subsection{Proof for proposition 3} 
The m-estimator for $ Q_3 $ is very similar to $ Q_2 $. We define the parameters, the m-estimator function, and $ g(\bm{\theta}) $ as:

\begin{equation}
	\begin{split}
		\bm{\theta}^T &= (\bm{\theta_1}, \bm{\theta_2}) \\
		\bm{\theta_1} &= (\theta_1^{pq})_{p=\{1, \cdots, (k-1) \}, q=\{1, \cdots, (r-1) \}} \\
		\bm{\theta_2} &= (\theta_2^{pqst})_{p, s = \{1, \cdots, k-1 \},  q, t = \{1, \cdots, r-1 \}} \\
		\Psi(X_i, Y_i, \mathbf{Z}_i; \bm{\theta}^T) &= 
		\begin{cases}
			R_{{\mathbb{I}(x=p)}_i}.R_{{\mathbb{I}(y=q)}_i} - \theta_1^{pq} \, \, \forall p \in \{ 1, \cdots, k-1 \}, q \in \{1, \cdots, r-1 \} \\
			(R_{{\mathbb{I}(x=p)}_i}.R_{{\mathbb{I}(y=q)}_i} \times R_{{\mathbb{I}(x=s)}_i}.R_{{\mathbb{I}(y=t)}_i}) - \theta_{2}^{pqst} \, \, \forall p, s \in \{ 1, \cdots, k-1 \}, q, t \in \{ 1, \cdots, r-1 \} \\
		\end{cases} \\
		g(\bm{\theta}) &= \bm{\theta_1}\bm{\wedge} \\
	\end{split}
\end{equation}

where $ \bm{\wedge} $ is defined as:
\begin{equation}
	\begin{split}
		\bm{\wedge} &= \bm{\Sigma}^{-1} \\
		\Sigma^{ps, qt} &= \mathbb{E}[(R_{\mathbb{I}(x=p)} R_{\mathbb{I}(y=q)})(R_{\mathbb{I}(x=s)}R_{\mathbb{I}(y=t)})] - \mathbb{E}[R_{\mathbb{I}(x=p)} R_{\mathbb{I}(y=q)}] \mathbb{E}[R_{\mathbb{I}(x=s)} R_{\mathbb{I}(y=t)}] \\
			    &= \theta_2^{pqst} - \theta_1^{pq} \theta_1^{st} \\
	\end{split}
\end{equation}

Solving the equation $ \sum_i \Psi(X_i, Y_i, \mathbf{Z}_i; \bm{\theta}^T) = 0 $ gives us:
\begin{equation}
	\begin{split}
		\theta_1^{pq} & = \mathbb{E}[R_{\mathbb{I}(x=p)} R_{\mathbb{I}(y=q)}] \\
		\theta_2^{pqst} &= \mathbb{E}[(R_{\mathbb{I}(x=p)} R_{\mathbb{I}(y=q)})(R_{\mathbb{I}(x=s)} R_{\mathbb{I}(y=t)})] \\
	\end{split}
\end{equation}

Similar to the last case as both $ \Psi $ and $ g(\bm{\theta}) $ are smooth functions with $ g(\bm{\theta}) = 0 $ under the null,

\begin{equation}
	\sqrt{n}g(\hat{\bm{\theta}}) \rightarrow \mathcal{N}(\bm{0}, \bm{\sigma}^2)
\end{equation}

The rest of the proof is very similar to the last case and we get the following values for the
matrices:
\begin{equation}
	\begin{split}
		A(\bm{\theta}) &= \mathbb{I}_{(k-1)(r-1)} \\
		B(\bm{\theta}) &= \begin{bmatrix}
			\bm{B_{00}} & \bm{B_{01}} \\
			\bm{B_{10}} & \bm{B_{11}} \\
				 \end{bmatrix} \\
		\bm{B_{00}} &= \bm{\Sigma} \\
		B_{10}^{lm, stuv} &= \mathbb{E}[R_{\mathbb{I}(x=l)} R_{\mathbb{I}(y=m)} R_{\mathbb{I}(x=s)} R_{\mathbb{I}(y=t)} R_{\mathbb{I}(x=u)} R_{\mathbb{I}(y=v)}] \\
		\bm{B_{01}} &= \bm{B_{10}}^T \\
		B_{11}^{lmpq, stuv} &= \mathbb{E}[(R_{\mathbb{I}(x=l)} R_{\mathbb{I}(y=m)})(R_{\mathbb{I}(x=p)} R_{\mathbb{I}(y=q)})(R_{\mathbb{I}(x=s)} R_{\mathbb{I}(y=t)})(R_{\mathbb{I}(x=u)} R_{\mathbb{I}(y=v)})] - \theta_2^{lmpq} \theta_2^{stuv} \\
		\frac{\partial}{\partial \theta} g(\bm{\theta}) &= \begin{bmatrix} \bm{\wedge} & \bm{0} \end{bmatrix} \\
		\bm{\sigma}^2 &= \mathbb{I}_{(k-1)(r-1)} \\
	\end{split}
\end{equation}

Hence, $ \sqrt{n} g(\bm{\theta}) $ has an asymptotic $ (k-1)(r-1) $ variate standard normal distribution. Since, $ Q_3 = \sum (\sqrt{n} g(\bm{\theta}))^2 $, $ Q_3 $ is asymptotically chi-square distributed with $ (k-1)(r-1) $ degrees of freedom.

\end{document}